\documentclass{article}
\usepackage{amsthm}

\usepackage{amsmath,amsfonts,bm}

\def\eqref#1{equation~\ref{#1}}

\def\1{\bm{1}}

\DeclareMathAlphabet{\mathsfit}{\encodingdefault}{\sfdefault}{m}{sl}
\SetMathAlphabet{\mathsfit}{bold}{\encodingdefault}{\sfdefault}{bx}{n}

\newcommand{\R}{\mathbb{R}}

\usepackage{hyperref}
\usepackage{url}
\usepackage{amssymb}
\usepackage{pifont}
\usepackage{multirow}
\usepackage{adjustbox}
\usepackage{booktabs}
\usepackage{siunitx}
\usepackage{subcaption}
\usepackage{wrapfig}
\usepackage{amsmath}
\usepackage{mathtools}
\usepackage{amsthm}
\usepackage{tabularx}
\usepackage{xcolor,colortbl}
\usepackage{pifont}
\usepackage{siunitx}
\usepackage{subcaption}
\usepackage{makecell}
\usepackage{enumitem}

\newcommand{\xmark}{\ding{55}}

\usepackage[capitalize,noabbrev]{cleveref}
\renewcommand{\eqref}[1]{(\ref{#1})}
\usepackage[preprint]{icml2026}

\theoremstyle{plain}
\newtheorem{theorem}{Theorem}[section]

\newtheorem{lemma}[theorem]{Lemma}

\theoremstyle{definition}

\theoremstyle{remark}

\usepackage[textsize=tiny]{todonotes}

\icmltitlerunning{LOTFormer: Doubly-Stochastic Linear Attention via Low-Rank Optimal Transport}

\begin{document}

\twocolumn[
  \icmltitle{LOTFormer: Doubly Stochastic Linear Attention via \\ Low Rank Optimal Transport}



  \icmlsetsymbol{equal}{*}

\begin{icmlauthorlist}
  \icmlauthor{Ashkan Shahbazi}{vandy}
  \icmlauthor{Chayne Thrash}{vandy}
  \icmlauthor{Yikun Bai}{vandy}
  \icmlauthor{Keaton Hamm}{uta}
  \icmlauthor{Navid NaderiAlizadeh}{duke}
  \icmlauthor{Soheil Kolouri}{vandy}
\end{icmlauthorlist}

\icmlaffiliation{vandy}{Department of Computer Science, Vanderbilt University, Nashville, Tennessee, USA}
\icmlaffiliation{uta}{Department of Mathematics, The University of Texas at Arlington, Arlington, Texas, USA}
\icmlaffiliation{duke}{Department of Biostatistics and Bioinformatics, Duke University, Durham, North Carolina, USA}

\icmlcorrespondingauthor{Ashkan Shahbazi}{your.email@vanderbilt.edu}
\icmlcorrespondingauthor{Soheil Kolouri}{soheil.kolouri@vanderbilt.edu}

\icmlkeywords{Machine Learning, ICML}

\vskip 0.3in
]


\printAffiliationsAndNotice{}  

\begin{abstract}

Transformers have proven highly effective across modalities, but standard softmax attention scales quadratically with sequence length, limiting long context modeling. Linear attention mitigates this by approximating attention with kernel feature maps, yet most attention mechanisms remain row normalized and can over concentrate mass on a few tokens, harming robustness and information flow. Doubly stochastic attention counteracts this by balancing token participation across both rows and columns, but existing approaches often add significant overhead. We propose LOTFormer, a linear time doubly stochastic attention mechanism derived from an optimal transport view of attention as a coupling between query and key measures. LOTFormer enforces a low rank transport plan by conditioning on a learnable pivot measure with small support. We solve two entropic transport problems, queries to pivot and pivot to keys, and compose them into a conditional coupling that is provably doubly stochastic, has rank at most $r \ll n$, and applies to values in $O(nr)$ time without forming the full $n \times n$ matrix. The pivot locations and masses are learned end-to-end. Across vision and text benchmarks, LOTFormer delivers strong accuracy efficiency tradeoffs when plugged into standard backbones including Swin, DeiT, and BERT.
\end{abstract}

\vspace{-0.1in}
\section{Introduction}

Transformers \citep{vaswani2017attention} have emerged as one of the most influential neural network architectures, achieving state‑of‑the‑art performance across a wide range of domains \citep{lin2022survey}, from natural language processing \citep{grattafiori2024llama} to computer vision \citep{khan2022transformers}, audio and speech processing \citep{gulati2020conformer}, multi‑modality \citep{liu2023visual}, protein folding \citep{jumper2021highly}, and bioinformatics \citep{dalla2025nucleotide}. At the core of Transformer architectures lies the \emph{attention} mechanism, which effectively models complex dependencies among tokens in a sequence. As the demand for models capable of reasoning over long contexts continues to grow, extending Transformers to larger context windows has become increasingly important. However, this goal is hindered by the quadratic computational complexity of the attention mechanism. This challenge has spurred a growing body of research aimed at reducing the computational burden of attention, with a prominent line of work focusing on \emph{linear attention} methods \citep{katharopoulos2020transformers,wang2020linformer, choromanski2020rethinking, shen2021efficientattention, chen2021skyformer, xiong2021nystromformer, meng2025polaformer}. Similarly, our primary goal is to design attention mechanisms with linear complexity in sequence length.

On the other hand, prior work has shown that the row-stochastic nature of attention matrices often leads to attention concentrating disproportionately on a few tokens, which can hinder effective information flow \citep{sander2022sinkformers,shahbazi2025espformer}. In the case of vision Transformers, this phenomenon---referred to as \emph{token overfocusing}---has been observed to degrade performance, and promoting broader token participation has been found to improve both robustness and accuracy \citep{guo2023robustifying}. Moreover, recent work demonstrates that mitigating overfocusing yields smoother, more interpretable attention maps and stronger performance on dense prediction and object discovery tasks \citep{darcet2024vision}.
One approach to mitigate overfocusing is to transform the row-stochastic attention matrix into a doubly-stochastic one, for example, using the Sinkhorn algorithm \citep{sinkhorn1964relationship}. While effective in enhancing robustness, this approach incurs even greater computational cost than standard quadratic attention, since Sinkhorn iterations are applied on top of the full quadratic attention map. More recently, \citet{shahbazi2025espformer} exploited the connection between transport plans \citep{peyre2019computational} and doubly-stochastic attention matrices and, building on advances in efficient transport plan computation \citep{liu2025estp}, introduced \textsc{ESPFormer}---a more computationally efficient formulation of doubly-stochastic attention. Nevertheless, \textsc{ESPFormer} remains quadratic during training due to its reliance on soft-sorting and exhibits super-linear complexity at inference.

In this work, we take a step further by investigating, for the first time, the feasibility of computing doubly-stochastic attention in \emph{linear} time. We propose a novel attention mechanism that, similar to \textsc{ESPFormer}, formulates attention as a transportation plan between the empirical measures defined by queries and keys. Our key innovation is the introduction of a learnable \emph{pivot measure} with a small support size $r \ll n$, which serves as an intermediate representation. Instead of directly computing the $n \times n$ transport plan, we factorize it by first solving for the optimal transport between queries and the pivot, and then between the pivot and keys. This construction, which we denote as LOTFormer, yields a low-rank decomposition of the attention matrix that simultaneously preserves its doubly-stochastic structure and reduces the computational complexity from quadratic to linear in the sequence length. Figure~\ref{fig:main} illustrates our framework for sample queries and keys under varying pivot sizes, i.e., different values of $r$.


\textbf{Our specific contributions} are summarized below:
\vspace{-.1in}
\begin{enumerate}[itemsep=0pt]
    \item We propose LOTFormer, a linear time attention operator that realizes a doubly stochastic attention matrix through a low rank optimal transport factorization via a learnable pivot measure, enabling DS normalization without forming an \(n\times n\) attention map.
    \item We introduce a controlled plug and play evaluation protocol for bidirectional encoder self attention, swapping only the attention modules in trained checkpoints while keeping all other parameters fixed.
    \item We demonstrate consistent ImageNet 1K gains on DeiT Tiny and Small, PVT Tiny, and Swin Tiny, matching or exceeding strong linear attention and doubly stochastic baselines under matched training and evaluation settings.
\end{enumerate}

\begin{figure*}[t!]
    \centering
    \includegraphics[width=\linewidth]{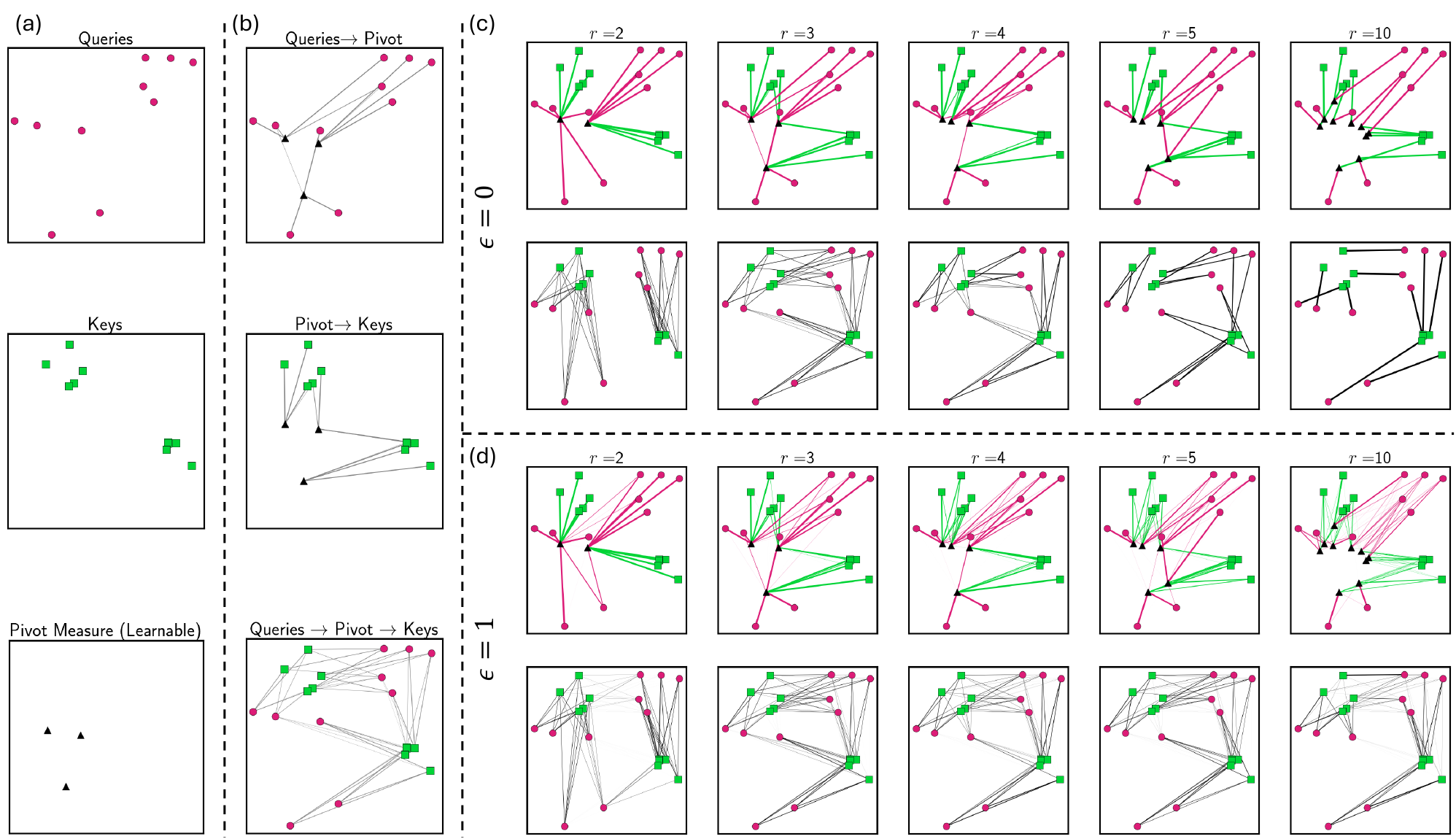}
    \caption{
Illustration of LOTFormer. 
\textbf{(a)} Queries (red circles), keys (green squares), and the learnable pivot measure (black triangles). 
\textbf{(b)} Factorization of the transport plan: instead of solving directly for a full $n\times n$ coupling between queries and keys, LOTFormer first computes transport from queries$\to$pivot and pivot$\to$keys, then composes them into a glued coupling. 
\textbf{(c--d)} Effective query–key couplings induced by the pivot measure for different pivot sizes ($r=2,3,4,5,10$). 
Top row of each block shows the mediated connections via pivots, and bottom row shows the resulting query–key couplings. 
\textbf{(c)} Without entropic regularization ($\varepsilon=0$), couplings are sharp and sparse. 
\textbf{(d)} With entropic regularization ($\varepsilon=1$), couplings become smoother and more diffuse.}
\vspace{-.2in}
    \label{fig:main}
\end{figure*}

\vspace{-.2in}
\section{Related Work}

\noindent\textbf{Linear Attention Mechanisms.} 
The original Transformer relies on quadratic time softmax attention, which becomes prohibitive as sequence length grows \citep{vaswani2017attention}. To address this, a broad literature has developed \emph{linear attention} methods that approximate or restructure attention to reduce computational and memory costs. Early approaches include \emph{Local Attention} \citep{local_attn} and \emph{Longformer} \citep{beltagy2020longformerlongdocumenttransformer}, which restrict interactions to local windows. More structured sparsity is explored by \emph{Reformer} \citep{Kitaev2020Reformer} and \emph{BigBird} \citep{NEURIPS2020_c8512d14}, which use locality sensitive hashing and block sparse patterns, respectively. A second family replaces the softmax kernel with low rank or kernelized approximations. \emph{Linear Transformer} \citep{katharopoulos2020transformers} and \emph{Performer} \citep{choromanski2020rethinking} use kernel feature maps to rewrite attention so that computation scales linearly with sequence length. \emph{Linformer} \citep{wang2020linformer} introduces low rank projections of key and value representations, while \emph{Nyströmformer} \citep{xiong2021nystromformer} applies Nyström approximations to kernel matrices. Related directions include \emph{Kernelized Attention} \citep{NEURIPS2021_c0f168ce}, \emph{Cosformer} \citep{zhen2022cosformer}, and \emph{Skyformer} \citep{chen2021skyformer}, which improve feature maps, projections, or bases. \emph{Synthesizer} \cite{tay2021synthesizer} departs from explicit query key interactions by learning synthetic attention weights. \emph{Informer} \citep{haoyietal-informer-2021} and \emph{Hedgehog} \citep{zhang2024the} target time series and structured settings, combining efficient approximations with task specific inductive biases. Collectively, these works replace dense quadratic interactions with kernel, low rank, or sparse mechanisms that aim to preserve expressivity while improving scalability.

\noindent\textbf{Doubly Stochastic Attention.}
Beyond efficiency, another line of work focuses on the \emph{stochastic structure} of attention matrices. Softmax enforces row stochasticity but does not normalize columns, which can yield imbalanced attention where a small set of tokens receives disproportionate mass. \emph{SinkFormer} \citep{sinkformer} enforces \emph{doubly stochastic} constraints via iterative Sinkhorn normalization, producing balanced transport like couplings between queries and keys. \emph{ESPFormer} \citep{shahbazi2025espformer} leverages expected sliced transport plans with annealed temperature schedules, aiming to improve both stability and performance by bridging soft and hard assignments.

More recently, \emph{Quantum DS Attention} \citep{born2025quantumdoublystochastictransformers} extends doubly stochastic constructions using quantum inspired parameterizations, exploiting structure from quantum stochastic matrices and unitary constraints while maintaining tractability. We view this as concurrent work that appeared during our study, and therefore we do not include a direct comparison against it in our experiments. Overall, these methods suggest that doubly stochasticity is more than a heuristic regularizer: it reframes attention as a transport problem with stronger normalization that can improve stability and interpretability of the learned couplings.

\noindent\textbf{Position of LOTFormer.}
\emph{LOTFormer} sits at the intersection of efficient and structured attention. It adopts an optimal transport view of attention, but applies it in linear time by forming a low rank coupling through a learnable pivot measure. This construction yields a doubly stochastic attention operator with explicit structure, while avoiding iterative normalization over full $n\times n$ maps. Moreover, the pivot factorization makes the transport structure transparent and interpretable, since interactions are mediated through a small set of representative support points. As a result, LOTFormer can be seen as a principled bridge between doubly stochastic normalizations and practical linear attentions.

\vspace{-.1in}
\section{Method}

We follow the notation in Appendix \ref{sec:notations} for optimal transport and doubly stochastic attention.
\begin{figure*}[t]
  \centering
  \includegraphics[width=0.9\textwidth]{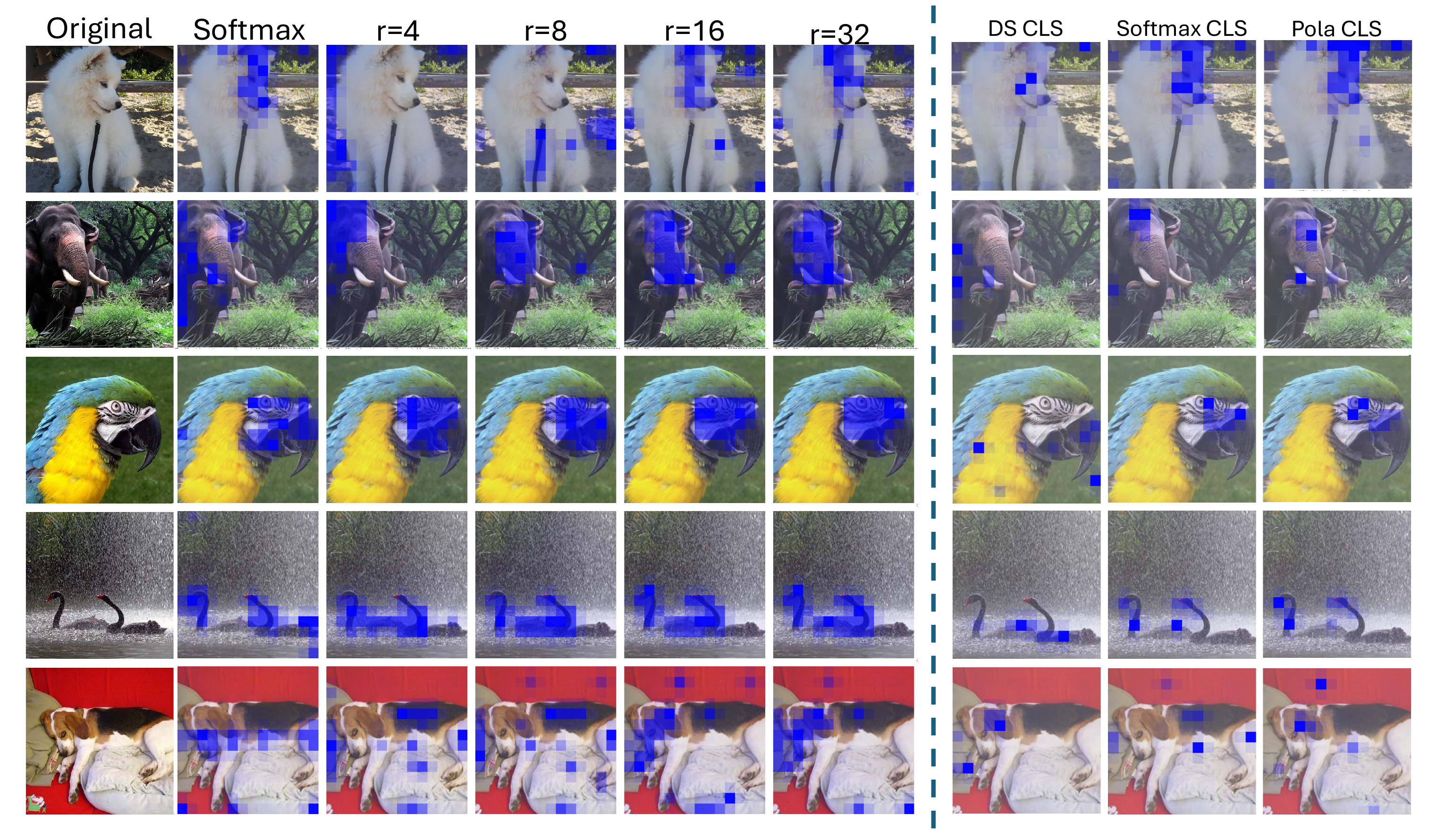}
  \vspace{-.1in}
  \caption{Patch-level visualizations of \texttt{[CLS]} attention. 
\textbf{(Left)} Comparison of standard \textsc{Softmax} attention with LOTFormer at different pivot sizes $r\!\in\!\{4,8,16,32\}$, showing how larger $r$ produces sharper, more object-centric maps. 
\textbf{(Right)} Effect of different \texttt{[CLS]} treatments (all without DWC): enforcing DS on \texttt{[CLS]} (\emph{full DS}) degrades global aggregation, whereas decoupling it via \texttt{[CLS]}–softmax restores broad coverage, and adding polarization (+Pola) further sharpens selectivity. 
The leftmost column preserves the original image for context, while the other columns use a neutral gray background to standardize contrast.}
  \label{fig:imagenet-r}
  \vspace{-.2in}
\end{figure*}



\noindent\textbf{Entropic Optimal Transport.}
Optimal transport (OT) constructs a nonnegative coupling matrix that links two discrete probability vectors while matching their prescribed marginals. Concretely, given $p\in\Delta^{n-1}$, $q\in\Delta^{m-1}$, and a similarity matrix $C\in\mathbb{R}^{n\times m}$, OT searches over the transportation polytope
\[
U(p,q)\;=\;\{\Gamma\in\mathbb{R}^{n\times m}_{+}:\ \Gamma\mathbf{1}=p,\ \Gamma^\top\mathbf{1}=q\},
\]
and selects a coupling that favors large similarity. Entropic regularization augments this selection with the entropy
\[
H(\Gamma)\;=\;-\sum_{i,j}\Gamma_{ij}\big(\log\Gamma_{ij}-1\big),
\]
which encourages diffuse couplings and yields a smooth, stable mapping from similarities to couplings. We will use this template as a normalization primitive, and later instantiate it with dot product similarities.

\noindent\textbf{OT as structured normalization.}
From the attention viewpoint, an OT coupling turns raw pairwise similarities into weights whose total mass allocation across rows and columns is controlled by the chosen marginals. This provides a principled alternative to purely row normalized attention and motivates constructing attention weights by solving OT problems induced by learned similarities.

\noindent\textbf{Setup.} 
Let $\{x_i\}_{i=1}^n \subset \mathbb{R}^{d_{\mathrm{in}}}$ be the set of $n$ input tokens to the attention block and
{\setlength{\abovedisplayskip}{2pt}\setlength{\belowdisplayskip}{2pt}
\[
\begin{aligned}
q_i = W_Q x_i \in \mathbb{R}^{d_k},
k_i = W_K x_i \in \mathbb{R}^{d_k},
v_i = W_V x_i \in \mathbb{R}^{d_v}.
\end{aligned}
\]
}

for $i\in [1,n]$, denote the queries, keys, and values accordingly, with learned projections $W_Q, W_K \in \mathbb{R}^{d_k\times d_{\mathrm{in}}}$ and $W_V \in \mathbb{R}^{d_v\times d_{\mathrm{in}}}$.
Let $Q=[q_1,\dots,q_n]^\top\!\in\mathbb{R}^{n\times d_k}$, $K=[k_1,\dots,k_n]^\top\!\in\mathbb{R}^{n\times d_k}$,
and $V=[v_1,\dots,v_n]^\top\!\in\mathbb{R}^{n\times d_v}$.

\noindent\textbf{Empirical measures and pivot.} In line with earlier approaches that establish a connection between attention mechanisms and optimal transport (OT) \citep{sander2022sinkformers,shahbazi2025espformer}, we represent queries and keys as empirical measures, namely,
{\setlength{\abovedisplayskip}{2pt}\setlength{\belowdisplayskip}{2pt}
\[
\begin{aligned}
\mu &= \sum_{i=1}^n \mathrm{p}^1_i\,\delta_{q_i},\qquad
\nu = \sum_{j=1}^n \mathrm{p}^2_j\,\delta_{k_j},
\end{aligned}
\]
}

$\text{typically } \mathrm{p}^1_i=\mathrm{p}^2_j=\tfrac{1}{n}$. We also introduce a \emph{learnable pivot measure} with support on $r\ll n$ points:
{\setlength{\abovedisplayskip}{1pt}\setlength{\belowdisplayskip}{1pt}
\[
\begin{aligned}
\sigma &= \sum_{t=1}^r \mathrm{p}^0_t\,\delta_{z_t}, \qquad
Z=\{z_t\}_{t=1}^r \subset \mathbb{R}^{d_k},\\
&\mathrm{p}^0_t>0,\quad \sum_{t=1}^r \mathrm{p}^0_t=1,\quad r\ll n.
\end{aligned}
\]
}
\noindent\textbf{Conditional OT $\mu\to \sigma \to \nu$.} An alternative but fundamental perspective on attention is that the entropically regularized transport plan between $\mu$ and $\nu$—that is, their joint measure—can be interpreted as a doubly stochastic attention matrix linking queries and keys. Our approach builds on this view by introducing a conditional transport plan, defined relative to a pivot measure, to construct such a coupling. By design, this formulation yields an attention matrix that is both low-rank and linearly decomposable. 

Here we define the entropic OT problems from $\mu$ to $\sigma$, and from $\sigma$ to $\nu$. Setting similarity to $c(q,k)=q^\top k$ we define
\[
C^{(1)}_{it}=z_t^\top q_i ,\qquad C^{(2)}_{jt}=z_t^\top k_j  
\]
With a shared regularization $\varepsilon>0$, solve the two entropy-regularized transport problems
\begin{align}
\Gamma^{(1)} &\in \arg\max_{\Gamma\in U(\mathrm{p}^0,\mathrm{p}^1)} \;\;\langle C^{(1)}, \Gamma \rangle \;+\; \varepsilon\, H(\Gamma), \label{eq:lot-ot1}\\
\Gamma^{(2)} &\in \arg\max_{\Gamma\in U(\mathrm{p}^0,\mathrm{p}^2)} \;\;\langle  C^{(2)},\Gamma\rangle \;+\; \varepsilon\, H(\Gamma), \label{eq:lot-ot2}
\end{align}
where $U(\alpha,\beta)=\{\Gamma\ge 0:\Gamma\mathbf{1}=\alpha,\;\Gamma^\top\mathbf{1}=\beta\}$ and
$H(\Gamma)=-\sum_{a,b}\Gamma_{ab}(\log \Gamma_{ab}-1)$. 

By classical OT theory, the maximizers of the above problems exist. Note that these transportation plans admit Sinkhorn scaling forms of:
\[
\Gamma^{(1)}=\mathrm{Diag}(u)\,K^{(1)}\,\mathrm{Diag}(v),\;\; K^{(1)}=\exp\!\big(Q Z^\top/\varepsilon\big),
\]
\[
\Gamma^{(2)}=\mathrm{Diag}(\tilde u)\,K^{(2)}\,\mathrm{Diag}(\tilde v),\;\; K^{(2)}=\exp\!\big(K Z^\top/ \varepsilon\big).
\]
As $\varepsilon \to 0$, the entropically regularized transport plans $\Gamma^{(1)}$ and $\Gamma^{(2)}$ converge to their corresponding optimal transport plans. Having obtained $\Gamma^{(1)}$, the plan from $\mu$ to $\sigma$, and $\Gamma^{(2)}$, the plan from $\sigma$ to $\nu$, we now construct the conditional plan, also referred to as the ``glued coupling,'' between $\mu$ and $\nu$.

\noindent\textbf{Glued coupling and doubly-stochastic attention.}
Define the glued (composed) coupling
\begin{equation}
\Gamma \;=\; (\Gamma^{(1)})^\top\, \mathrm{Diag}(\sigma)^{-1}\, \Gamma^{(2)} \;\in\; \mathbb{R}^{n\times n}.
\end{equation}
Then $\Gamma\mathbf{1}=\mu$ and $\Gamma^\top\mathbf{1}=\nu$, i.e., $\Gamma$ is a transportation plan between $\mu$ and $\nu$.
In the common balanced case $\mu=\nu=\tfrac{1}{n}\mathbf{1}$, we set
{\setlength{\abovedisplayskip}{2pt}\setlength{\belowdisplayskip}{2pt}
\[
\begin{aligned}
A \;=\; \Gamma
&:= (\Gamma^{(1)})^\top \,\mathrm{Diag}(\sigma)^{-1}\, \Gamma^{(2)},\\
s.t.\quad A\mathbf{1} &= \mathbf{1}, \qquad \mathbf{1}^\top A = \mathbf{1}^\top .
\end{aligned}
\]
}
so $A$ is \emph{doubly stochastic}. The LOTFormer head output is\\
{\setlength{\abovedisplayskip}{2pt}\setlength{\belowdisplayskip}{2pt}
\begin{equation}
\label{eq:lotattn}
\begin{aligned}
\mathrm{LOTAttn}(Q,K,V)
= A V \\
= (\Gamma^{(1)})^\top\!\Big(\mathrm{Diag}(\sigma)^{-1}\,(\Gamma^{(2)}V)\Big).
\end{aligned}
\end{equation}
}
\noindent\textbf{Rank and complexity.}
Write
\[
A \;=\; (\,(\Gamma^{(1)})^\top\,\mathrm{Diag}(\sigma)^{-1/2})\;(\mathrm{Diag}(\sigma)^{-1/2}\Gamma^{(2)}),
\]
hence $\mathrm{rank}(A)\le r$ (Note, in default setting $r\leq n$, we can show  $\mathrm{rank}(A)=r$). 
We may avoid forming $A$ itself by using its factorized form, by first computing $Y = \Gamma^{(2)}V\in\R^{r\times d_v}$, then $Z' = \mathrm{Diag}(\sigma)^{-1}Y$, and finally $O = \Gamma^{(1)}Z'\in\R^{n\times d_v}$.
Each Sinkhorn iteration uses matrix–vector scalings with kernels $K^{(1)}=\exp(Q Z^\top/\varepsilon)$ and
$K^{(2)}=\exp(Z K^\top/\varepsilon)$, costing $O(nr)$ per iteration (plus $O(nd_k r)$ to form $Q Z^\top$ and $Z K^\top$ once per update).
For fixed $r\ll n$ and modest iteration count $T$, the per-head complexity is $O\!\big(nd_k r + T\,nr\big)$, i.e., \emph{linear in $n$}.

\noindent\textbf{Learning.}
The pivot locations $Z=\{z_t\}$, masses $\sigma\in\Delta^{r-1}$, and projections $(W_Q,W_K,W_V)$ are trained end-to-end by
backpropagating through the shared-$\varepsilon$ Sinkhorn problems \eqref{eq:lot-ot1}–\eqref{eq:lot-ot2} and the glued composition. Note that for multi-head attention, each head has its own $(Z,\sigma)$.

For completeness, we provide a short mathematical explanation of why doubly stochastic constraints are incompatible with standard causal masking except for degenerate solutions. See Appendix ~\ref{app:causal_scope}.

\noindent\textbf{Why not the optimal Low-Rank OT?} 
Our construction produces an attention matrix that is, by design, low-rank due to the introduction of the pivot measure. However, this should not be conflated with the \emph{optimal low-rank transport plan} studied in prior work \cite{scetbon2022lowrank,scetbon2021sinkhorn}. In Low-Rank OT, the support of the low-rank factors is optimized directly to minimize the transportation cost, yielding a low-rank approximation that is provably optimal in that restricted class. By contrast, in our setting, the pivot measure is learned end-to-end as part of the attention mechanism, but not optimized explicitly for transport cost minimization. This distinction naturally raises the question of why not simply employ Low-Rank OT instead? The key reason is computational: computing Low-Rank OT still requires access to the full $n\times n$ cost matrix between $\mu$ and $\nu$ as input, which entails quadratic time and memory. This completely undermines the goal of linear-time attention. In contrast, our conditional transport construction never forms the full cost matrix, but instead only evaluates $QZ^\top$ and $ZK^\top$, each of size $n\times r$ with $r\ll n$, achieving linear complexity in $n$ while preserving double-stochasticity. 


\vspace{-0.1in}
\section{Experiments\protect\footnotemark}
\footnotetext{Our code and implementation can be found \href{https://example.com}{here}.}

\subsection{Runtime and Wall-Clock Analysis}

We study the computational efficiency of \textsc{LOTFormer}, a doubly–stochastic \emph{linear-time} attention, against representative quadratic and linear baselines. The quadratic  group includes \emph{Softmax}, \emph{ESPFormer (SoftSort)}, \emph{ESPFormer (HardSort)}, and \emph{Sinkformer}. The linear group includes \emph{Performer}, \emph{Nyströmformer}, \emph{MobiAttention}, and \emph{PolaFormer}. To expose the speed/accuracy control that low-rank methods afford, we also plot \textsc{LOTFormer} at ranks $r\!\in\!\{16,32,64,128,256\}$; $r{=}64$ is shown with a solid line and the other ranks with dashed lines.

Figure~\ref{fig:lotformer-two-panel} reports forward-pass runtime (ms/iteration) versus sequence length on a log–log scale for $N\in\{2^{9:17}\}$. The left panel compares quadratic methods to \textsc{LOTFormer}; the right panel compares linear-time baselines to the same \textsc{LOTFormer} curves. As expected, Softmax and the DS variants exhibit near–quadratic growth in $N$, whereas \textsc{LOTFormer} and the linear baselines grow roughly linearly. Across \textsc{LOTFormer} ranks, lower $r$ yields faster curves and higher $r$ yields slower curves, illustrating a clear rank–time trade-off. We include an appendix note discussing how LOTFormer differs from FlashAttention at the kernel level, and why our current implementation is an unfused PyTorch baseline rather than a fused attention kernel. See Appendix ~\ref{app:flashattention} for the scope disclaimer.

\begin{figure*}[t]
  \centering
  \includegraphics[width=\textwidth]{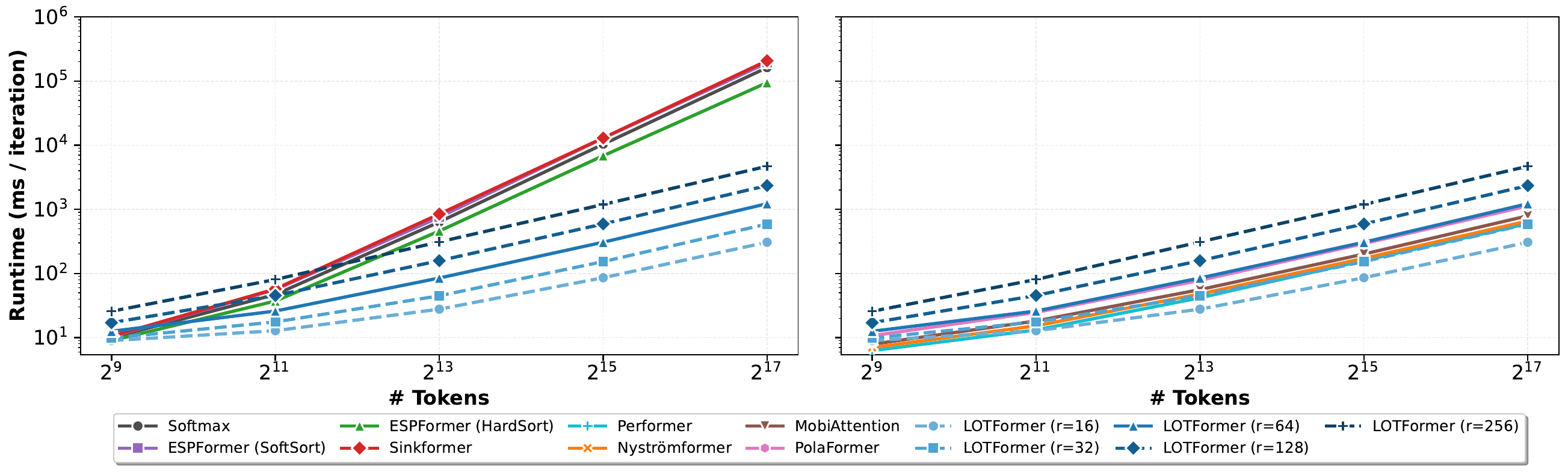}
  \caption{\textbf{Runtime scaling with sequence length.}
  Forward-pass runtime (ms/iteration) vs.\ $N$ on a log–log scale for quadratic methods (left) and linear methods (right). 
  \textsc{LOTFormer} is shown at ranks $r\!\in\!\{16,32,64,128,256\}$ in both panels (solid for $r{=}64$, dashed otherwise). 
  Points denote measured values at $N\!\in\!\{2^{9:17}\}$. \vspace{-0.15in}}
  \label{fig:lotformer-two-panel}
\end{figure*}


\vspace{-0.1in}
\subsection{ImageNet-1K}

Next, we evaluate LOTFormer on ImageNet 1K at $224^2$ input resolution with DeiT Tiny \citep{touvron2021deit}, DeiT Small, PVT Tiny \citep{wang2021pyramid}, and Swin Tiny \citep{liu2021Swin}, and compare against recent efficient attention variants (Table~\ref{tab:imagenet_main_merged}). LOTFormer achieves the best Top 1 within the DeiT Tiny, DeiT Small, and PVT Tiny blocks, improving over the corresponding Softmax baselines from $72.2$ to $74.8$, from $79.8$ to $80.1$, and from $75.1$ to $79.6$, respectively, with comparable parameter counts and FLOPs. On Swin Tiny, LOTFormer attains $81.8$ Top 1, placing it as the top performer in most blocks and within the top three in the remaining block at the same resolution.

\noindent\textbf{\texttt{[CLS]} under DS attention.}\label{sec:cls-heads}
A practical issue when using doubly stochastic (DS) attention in ViTs is the \texttt{[CLS]} token: it must pool information from all image tokens, while DS constraints can, in principle, over-concentrate attention and weaken global aggregation. In the extreme case (vanishing entropy regularization and $r=n$), a DS attention matrix may approach a permutation, so each query attends to a single key, and \texttt{[CLS]} would aggregate from only one token. Although entropy regularization and low-rank structure mitigate this in practice, the tension remains conceptually.
To preserve the global pooling role of \texttt{[CLS]}, we keep a standard softmax aggregator for the \texttt{[CLS]} row and apply DS constraints only among the image tokens:
{
\setlength{\abovedisplayskip}{5pt}
\setlength{\belowdisplayskip}{7pt}
\setlength{\abovedisplayshortskip}{4pt}
\setlength{\belowdisplayshortskip}{6pt}
\begin{equation}
\alpha_{\mathrm{cls},j}
= \frac{\exp\!\big(\beta\,\langle q_{\mathrm{cls}}, k_j\rangle\big)}
        {\sum_{l}\exp\!\big(\beta\,\langle q_{\mathrm{cls}}, k_l\rangle\big)},\qquad \beta>0,
\end{equation}
}
while all non-\texttt{[CLS]} rows follow the DS attention rule. 

\noindent\textbf{\texttt{[CLS]} Polarization.} Following \cite{meng2025polaformer}, we split a scalar \(x\) into its positive and negative components via \(x^{+}=\mathrm{ReLU}(x)\) and \(x^{-}=\mathrm{ReLU}(-x)\). In LOTFormer, we optionally restrict this operation to the \texttt{[CLS]} token, so that only \texttt{[CLS]}-to-token (and token-to-\texttt{[CLS]}) interactions are polarized while all other token-to-token interactions remain unchanged, by defining:
\vspace{-0.08in}
\begin{equation}
\label{eq:cls_score_tilde}
\begin{aligned}
\tilde{s}_{\mathrm{cls},j}
&= \Big((q_{\mathrm{cls}}^+)^\top k_j^+ + (q_{\mathrm{cls}}^-)^\top k_j^-\Big)^{p_s} \\
&\quad + \Big((q_{\mathrm{cls}}^+)^\top k_j^- + (q_{\mathrm{cls}}^-)^\top k_j^+\Big)^{p_o},
\end{aligned}
\end{equation}
followed by $\alpha_{\mathrm{cls},j}=\mathrm{softmax}_j(\tilde{s}_{\mathrm{cls},j})$.
By contrast, PolaFormer applies the same polarized logit construction to all rows, followed by per-row softmax normalization.
In our setting, only the \texttt{[CLS]} row uses softmax (with optional polarization), while the remaining rows remain DS.

\noindent\textbf{Reporting protocol.}
Table~\ref{tab:imagenet_main_merged} reports the best configuration per method at $224^2$ resolution for each backbone family. Several efficient attention baselines, most notably PolaFormer, include a depthwise convolution (DWC) token mixing layer as part of their reported architecture. To ensure fair comparisons, we run LOTFormer under the same setting when comparing against such methods, and we additionally report matched results without DWC. Table~3 isolates the effect of DWC under otherwise identical training and evaluation settings, together with ablations on the \texttt{[CLS]} aggregator and \texttt{[CLS]} only polarization. Across backbones, the dominant gain for doubly stochastic style attention comes from enabling a dedicated \texttt{[CLS]} softmax aggregator, while \texttt{[CLS]} only polarization provides a smaller but consistent improvement.

\noindent\textbf{Qualitative comparison of \texttt{[CLS]} attention.} In Fig.~\ref{fig:imagenet-r}, the \emph{left panel} shows the source image (first column), the standard \emph{softmax} \texttt{[CLS]} map (second), and \emph{LOTFormer} \texttt{[CLS]} maps for increasing pivot size $r$ (remaining columns, left$\to$right).
For small $r$, LOTFormer exhibits a \emph{clustering} effect in which attention splits across several coarse regions; as $r$ increases, the maps become progressively \emph{more object-centric}, concentrating on semantically relevant parts while retaining sufficient global coverage. The \emph{right panel} contrasts \texttt{[CLS]} treatments—full DS, \texttt{[CLS]}–softmax (tokens-only DS), and +Pola, highlighting how each modulates selectivity and spatial support.

\definecolor{lotblue}{RGB}{241,246,255}
\newcommand{\hicell}[1]{\cellcolor{lotblue}{#1}}

\begin{table}[t]
\centering
\caption{ImageNet 1K results at $224^2$ input resolution across four backbones. We report parameter count, FLOPs, and Top 1 accuracy. Best Top 1 within each backbone block is in \textbf{bold}. \vspace{-.1in}}
\label{tab:imagenet_main_merged}
\footnotesize
\setlength{\tabcolsep}{2.0pt}
\renewcommand{\arraystretch}{1.02}

\begin{tabular}{@{}%
>{\raggedright\arraybackslash}p{0.42\columnwidth}
>{\centering\arraybackslash}p{0.12\columnwidth}
>{\centering\arraybackslash}p{0.13\columnwidth}
>{\centering\arraybackslash}p{0.13\columnwidth}
>{\centering\arraybackslash}p{0.13\columnwidth}@{}}
\toprule
Method & Input & Params & FLOPs & Top 1 \\
\midrule

\multicolumn{5}{@{}c@{}}{\textbf{DeiT Tiny}} \\
\midrule
Softmax$_{\text{\citep{touvron2021deit}}}$                       & $224^2$ & 5.7M  & 1.1G & 72.2 \\
EfficientAttn$_{\text{\citep{shen2021efficientattention}}}$
                 & $224^2$ & 5.7M  & 1.1G & 70.2 \\
HydraAttn$_{\text{\citep{bolya2022hydraattentionefficientattention}}}$                      & $224^2$ & 5.7M  & 1.1G & 68.3 \\
EnhancedAttn$_{\text{\citep{cai2022efficientvit}}}$                   & $224^2$ & 5.8M  & 1.1G & 72.9 \\
FLattenAttn$_{\text{\citep{han2023flatten}}}$                     & $224^2$ & 6.1M  & 1.1G & 74.1 \\
AngularAttn$_{\text{\citep{you2024castlingvitcompressingselfattentionswitching}}}$                     & $224^2$ & 5.7M  & 1.1G & 70.8 \\
MobiAttn$_{\text{\citep{pmlr-v235-yao24c}}}$                       & $224^2$ & 5.7M  & 1.2G & 73.3 \\
PolaFormer$_{\text{\citep{meng2025polaformer}}}$                     & $224^2$ & 6.1M  & 1.2G & 74.6 \\
Sinkformer$_{\text{\citep{sinkformer}}}$                     & $224^2$ & 5.7M  & 1.2G & 70.2 \\
ESPFormer$_{\text{\citep{shahbazi2025espformer}}}$ & $224^2$ & 5.7M  & 1.2G & 70.1 \\
\hicell{LOTFormer}  & \hicell{$224^2$} & \hicell{5.8M} & \hicell{1.2G} & \hicell{\textbf{74.8}} \\
\midrule

\multicolumn{5}{@{}c@{}}{\textbf{DeiT Small}} \\
\midrule
Softmax$_{\text{\citep{touvron2021deit}}}$                      & $224^2$ & 22M & 4.6G & 79.8 \\
PolaFormer$_{\text{\citep{meng2025polaformer}}}$                  & $224^2$ & 23M & 4.8G & 79.8 \\
\hicell{LOTFormer} & \hicell{$224^2$} & \hicell{22.3M} & \hicell{4.8G} & \hicell{\textbf{80.1}} \\
\midrule

\multicolumn{5}{@{}c@{}}{\textbf{PVT Tiny}} \\
\midrule
Softmax$_{\text{\citep{wang2021pyramid}}}$                          & $224^2$ & 13.0M & 1.9G & 75.1 \\
PolaFormer$_{\text{\citep{meng2025polaformer}}}$             & $224^2$ & 12.0M & 2.0G & 78.8 \\
\hicell{LOTFormer}     & \hicell{$224^2$} & \hicell{12.2M} & \hicell{2.1G} & \hicell{\textbf{79.6}} \\
\midrule

\multicolumn{5}{@{}c@{}}{\textbf{Swin Tiny}} \\
\midrule
Softmax$_{\text{\citep{liu2021Swin}}}$                           & $224^2$ & 28.0M & 4.4G & 81.2 \\
HydraAttn$_{\text{\citep{bolya2022hydraattentionefficientattention}}}$                 & $224^2$ & 29.0M & 4.5G & 80.7 \\
EfficientAttn$_{\text{\citep{shen2021efficientattention}}}$             & $224^2$ & 29.0M & 4.5G & 81.0 \\
AngularAttn$_{\text{\citep{you2024castlingvitcompressingselfattentionswitching}}}$         & $224^2$ & 29.0M & 4.5G & 79.4 \\
EnhancedAttn$_{\text{\citep{cai2022efficientvit}}}$              & $224^2$ & 29.0M & 4.5G & \underline{81.8} \\
FLattenAttn$_{\text{\citep{han2023flatten}}}$               & $224^2$ & 29.0M & 4.5G & \underline{82.1} \\
PolaFormer$_{\text{\citep{meng2025polaformer}}}$                & $224^2$ & 29.0M & 4.5G & \textbf{82.6} \\
\hicell{LOTFormer}        & \hicell{$224^2$} & \hicell{28.0M} & \hicell{4.5G} & \underline{\hicell{81.8}} \\
\bottomrule
\end{tabular}
\vspace{-0.3in}
\end{table}

\subsection{Long Range Arena}
We evaluate LOTFormer on the \emph{Long Range Arena (LRA)} benchmark, which stresses long context modeling across five modalities: structured parsing (ListOps), text classification (Text), image classification (Image), document retrieval (Retrieval), and synthetic reasoning (Pathfinder). To make the comparison as controlled as possible, we follow the same LRA training and preprocessing recipe used by Skyformer, including the two layer Transformer architecture and matching optimization settings.

\textbf{Results and comparison to linear attention and state space models.}
Table~\ref{tab:lra_results} summarizes LRA performance for LOTFormer and a broad set of efficient baselines. We focus on recent linear time attention methods and modern state space models, and include strong SSM references such as DeltaNet \citep{yang2025parallelizinglineartransformersdelta} and Mamba2 \citep{dao2024transformersssmsgeneralizedmodels}, which are competitive alternatives on long sequences. Following common LRA recipes, we report LOTFormer both in the plain Transformer form and with an DWC token mixing layer, since several established baselines rely on this component as part of their model design. For a direct and controlled comparison, we adopt the same Transformer configuration and optimization choices used in Skyformer\citep{chen2021skyformer}. Across tasks, LOTFormer matches or improves upon prior baselines under the same protocol and remains competitive with strong SSM alternatives. When DWC is enabled, LOTFormer attains its best average performance, indicating that low rank doubly stochastic normalization provides a practical linear time route to stable long context modeling on LRA.

\newcommand{\DWCmark}{\textsuperscript{\scriptsize$\dagger$}}
\newcommand{\pmstd}{\;\pm\;}

\begin{table}[t]
\centering
\caption{LRA results. LOTFormer numbers are means over three runs. Best per column in bold. $\dagger$ indicates +DWC. \vspace{-.1in}}
\label{tab:lra_results}
\footnotesize
\setlength{\tabcolsep}{1.2pt} 
\renewcommand{\arraystretch}{1.15}

\begin{tabular*}{\columnwidth}{@{\extracolsep{0pt}}%
>{\raggedright\arraybackslash}p{0.25\columnwidth}
*{6}{>{\centering\arraybackslash}p{0.116\columnwidth}}@{}}
\toprule
Model & Text & List & Ret & Path & Img & Avg \\
\midrule
Transformer   & 61.55 & 38.71 & 80.93 & 70.39 & 39.14 & 58.14 \\
LocalAttn     & 52.98 & 15.82 & 53.39 & 66.63 & 41.46 & 46.06 \\
LinearTrans.  & 65.90 & 16.13 & 53.09 & 75.30 & 42.34 & 50.55 \\
Reformer      & 56.10 & 37.27 & 53.40 & 68.50 & 38.07 & 50.67 \\
Performer     & 65.40 & 18.01 & 53.82 & \textbf{77.05} & 42.77 & 51.41 \\
Synthesizer   & 61.68 & 36.99 & 54.67 & 69.45 & 41.61 & 52.88 \\
Longformer    & 62.85 & 35.63 & 56.89 & 69.71 & 42.22 & 53.46 \\
Informer      & 62.13 & 37.05 & 79.35 & 56.44 & 37.86 & 54.57 \\
Bigbird       & 64.02 & 36.05 & 59.29 & 74.87 & 40.83 & 55.01 \\
Linformer     & 57.29 & 36.44 & 77.85 & 65.39 & 38.43 & 55.08 \\
Kernelized    & 60.02 & 38.46 & 82.11 & 69.86 & 32.63 & 56.62 \\
Cosformer     & 63.54 & 37.20 & 80.28 & 70.00 & 35.84 & 57.37 \\
Nystrom       & 62.36 & 37.95 & 80.89 & 69.34 & 38.94 & 57.90 \\
Skyformer     & 64.70 & 38.69 & 82.06 & 70.73 & 40.77 & 59.39 \\
Hedgehog      & 64.60 & 37.15 & 82.24 & 74.16 & 40.15 & 59.66 \\
PolaFormer$_{\alpha=3}$          & 59.50 & \textbf{38.80} & 79.50 & 67.50 & 42.00 & 57.46 \\
PolaFormer$_{\alpha=3}$\DWCmark  & 73.06 & 37.35 & 80.50 & 70.53 & 42.15 & 60.72 \\
\midrule
DeltaNet  & 74.34 & 38.16 & \textbf{83.51} & 73.37 & 36.55 & 61.19 \\
Mamba2  & \textbf{78.11} & 37.86 & 82.66 & 70.14 & 42.28 & 62.21 \\
\midrule
LOTFormer              & 65.2 & 38.5 & 80.4 & 73.2 & 45.7 & 60.60 \\
LOTFormer\DWCmark      & 71.1 & 38.5 & 80.9 & 69.9 & \textbf{54.1} & \textbf{62.90} \\
\bottomrule
\end{tabular*}

\vspace{-0.25in}
\end{table}

\begin{table*}[t]
\centering
\tiny
\setlength{\tabcolsep}{2.6pt} 
\renewcommand{\arraystretch}{0.95}
\caption{Side-by-side ablation on DeiT-Tiny / ImageNet-1K. “Pola type” indicates where polarization is applied:
None, All tokens (PolaFormer), or \texttt{[CLS]}-only (LOTFormer). $\Delta$ is computed within each block relative to the previous row.\vspace{-.1in}}
\label{tab:two-part-ablation-polatype}
\resizebox{\linewidth}{!}{%
\begin{tabular}{@{}ccccc c ccccc c ccccc@{}}
\toprule
\multicolumn{5}{c}{\textbf{PolaFormer}} & & 
\multicolumn{5}{c}{\textbf{FlattenAttention}} & &
\multicolumn{5}{c}{\textbf{LOTFormer ($r{=}32$)}} \\
\cmidrule(lr){1-5}\cmidrule(lr){7-11}\cmidrule(l){13-17}
\texttt{[CLS]}-Softmax & Pola type & DWC & Top-1 (\%) & $\Delta$
& & \texttt{[CLS]}-Softmax & Pola type & DWC & Top-1 (\%) & $\Delta$
& & \texttt{[CLS]}-Softmax & Pola type & DWC & Top-1 (\%) & $\Delta$ \\
\midrule
\checkmark & All  & \xmark     & 61.9           & 0.0
& & \checkmark & None & \xmark     & 71.8 & 0.0
& & \xmark     & None & \xmark     & 68.2           & 0.0 \\
\checkmark & All  & \checkmark & 74.6           & {+12.7}
& & \checkmark & None & \checkmark & 74.1 &     +2.3
& & \checkmark & None & \xmark     & 73.2           & {+5.0} \\
    &      &             &                 &
& &        &      &             &                 &
& & \checkmark & \texttt{[CLS]}  & \xmark     & 73.6           & {+0.4} \\
    &      &             &                 &
& &        &      &             &                 &
& & \checkmark & \texttt{[CLS]}  & \checkmark & \textbf{74.8}  & {+1.2} \\
\bottomrule
\end{tabular}%
}
\vspace{-0.15in}
\end{table*}

\begin{table}[t]
  \centering
  \caption{Plug-and-Play conversion on fine-tuned GLUE using \textsc{BERT}-base.
  We replace encoder bidirectional self-attention blocks only while keeping other parameters unchanged.
  Percent recovery is computed relative to the fine-tuned baseline.\vspace{-0.08in}}
  \label{tab:glue_plug_and_play}
  \setlength{\tabcolsep}{2.5pt}
  \renewcommand{\arraystretch}{1.05}
  \scriptsize
  \resizebox{0.95\columnwidth}{!}{%
  \begin{tabular}{@{} c|c|cccc @{}}
    \toprule
    Task & \textsc{BERT}$_{\text{FT}}$ & T2R & T2R\textendash HH & Hedgehog & \textbf{Ours} \\
    \midrule
    CoLA        & 58.8 & 43.6 & 56.9 & 59.2 & 59.1 \\
    SST2        & 93.2 & 87.7 & 90.9 & 92.6 & 92.5 \\
    MRPC        & 90.2 & 83.0 & 89.1 & 90.1 & 90.3 \\
    STS\textendash B & 88.8 & 78.6 & 77.7 & 87.4 & 89.7 \\
    QQP         & 91.0 & 86.7 & 90.0 & 91.0 & 90.5 \\
    MNLI        & 84.7 & 78.9 & 77.4 & 82.6 & 83.0 \\
    QNLI        & 91.3 & 84.6 & 84.5 & 89.6 & 90.2 \\
    RTE         & 68.2 & 54.1 & 56.3 & 69.3 & 68.5 \\
    \midrule
    Recover (\%)& 100.0 & 88.9 & 93.5 & 99.3 & \textbf{99.7} \\
    \bottomrule
  \end{tabular}%
  }
\vspace{-0.12in}
\end{table}

\vspace{-0.15in}
\subsection{Plug-and-Play on Text benchmarks}
\vspace{-0.05in}
\label{subsec:text_benchmarks}
A central goal of this work is to make doubly stochastic attention usable \emph{as a drop-in replacement} in existing Transformer checkpoints. To this end, our Plug-and-Play evaluation starts from a \emph{trained} model and performs a surgical conversion: we replace \textbf{only the encoder bidirectional self-attention operator} with a candidate attention mechanism, while keeping the rest of the network \emph{bitwise unchanged} (token embeddings, FFNs, layer norms, residual connections, positional encodings, and task-specific heads). This protocol treats attention as a modular component and isolates its contribution without confounding changes from retraining.
For encoder--decoder models, we leave the decoder causal self-attention and encoder--decoder cross-attention untouched, and convert only the encoder self-attention blocks.

\noindent\textbf{Neural Machine Translation.}
We test drop-in conversion on IWSLT'14 De$\rightarrow$En~\cite{cettolo2014iwslt} using two standard backbones implemented in \texttt{fairseq}~\cite{ott2019fairseq}: the Transformer and DiffTransformer~\cite{ye2025differential}, each with a 6-layer encoder and 6-layer decoder. We compare LOTFormer to two recent doubly stochastic baselines, Sinkformer and ESPFormer, under identical conversion rules.
Our evaluation has two stages. First, we train each backbone for 25 epochs and then apply \emph{Plug-and-Play} conversion by swapping \textbf{encoder bidirectional self-attention} with LOTFormer, ESPFormer, or Sinkformer, and evaluating \emph{with no additional training}. This directly measures whether the attention operator can be deployed as a compatible module in an existing checkpoint. Second, we optionally allow a \emph{Fine-Tune Boost} phase, fine-tuning the converted model for 10 additional epochs to quantify the headroom available once the model adapts to the new attention. Results are reported in Table~\ref{tab:nmt_transformer_vs_difftransformer}.

\noindent\textbf{GLUE conversion on fine-tuned \textsc{BERT}.}
For GLUE \citep{wang2019gluemultitaskbenchmarkanalysis}, we follow the same fine-tuned conversion protocol used in Hedgehog \citep{zhang2024hedgehog} and Transformer-to-RNN (T2R) \citep{kasai2021finetuningpretrainedtransformersrnns}: starting from a task fine-tuned \textsc{BERT}-base \citep{devlin2019bertpretrainingdeepbidirectional} checkpoint, we swap all encoder bidirectional self-attention blocks with LOTFormer, run a short conversion stage using an distillation objective, and then fine-tune on the original task to measure recovery relative to softmax. 

\begin{table}[t]
  \centering
  \caption{Plug and Play and Fine Tune performance on IWSLT'14 De$\rightarrow$En (median over 4 runs). $^{\star}$ marks Plug and Play when swapping in an attention module different from the base model's own. $\Delta$ is measured relative to the base model within each block.\vspace{-0.1in}}
  \label{tab:nmt_transformer_vs_difftransformer}
  \setlength{\tabcolsep}{5pt}
  \renewcommand{\arraystretch}{1.2}
  \small

  \begin{tabular}{@{} l c c c @{}}
    \toprule
        Model & Plug and Play & Fine Tune & $\Delta$ \\
    \midrule
    \multicolumn{4}{@{}c@{}}{\emph{Base: Transformer}} \\
    \midrule
    Transformer      & \textbf{33.40}$^{\phantom\star}$   & 34.61          & --- \\
    Sinkformer       & 33.36$^{\star}$ & 34.61          & +0.00 \\
    ESPFormer        & 33.38$^{\star}$ & 34.64          & +0.03 \\
    LOTFormer (ours) & 33.29$^{\phantom\star}$            & \textbf{34.72} & \textbf{+0.11} \\
    \midrule
    \multicolumn{4}{@{}c@{}}{\emph{Base: DiffTransformer}} \\
    \midrule
    Transformer      & \textbf{33.85}$^{\phantom\star}$   & 34.78          & --- \\
    Sinkformer       & 33.67$^{\star}$ & 34.81          & +0.03 \\
    ESPFormer        & 33.72$^{\star}$ & 34.83          & +0.05 \\
    LOTFormer (ours) & 33.42$^{\phantom\star}$            & \textbf{34.91} & \textbf{+0.13} \\
    \bottomrule
  \end{tabular}
  \vspace{-0.25in}
\end{table}

\vspace{-0.1in}
\subsection{Ablation Studies}
\vspace{-0.08in}
\label{subsec:ablations}

Tables~\ref{tab:lotformer_ablation_vit} and~\ref{tab:lotformer_r_ablation_deit} study key design choices for LOTFormer under two representative settings.

\noindent\textbf{Sinkhorn settings and the reference measure on ViT.} Table~\ref{tab:lotformer_ablation_vit} varies the number of Sinkhorn iterations $T$, the entropic strength $\varepsilon$, and the choice of reference measure $\mathsf{ref}$, either fixed or learnable and initialized to uniform. Experiments use ImageNet 100 with a reduced ViT backbone that has half the hidden dimension.
For tighter regularization with $\varepsilon{=}0.1$, accuracy improves as $T$ increases, which is consistent with needing more iterations to approach a well balanced coupling. For $\varepsilon{=}1$, performance saturates earlier and the learnable $\mathsf{ref}$ is consistently stronger, reaching its best value at $T{=}10$. Sweeping $\varepsilon$ at fixed $T{=}10$ shows a clear optimum near $\varepsilon{=}1$, with degraded accuracy at both extremes. Very small $\varepsilon$ is harder to optimize under a fixed iteration budget, while very large $\varepsilon$ over smooths the transport plan. Overall, making $\mathsf{ref}$ learnable matches or improves on a fixed reference across all settings, and we use it in our strongest configurations.

\noindent\textbf{Pivot size $r$ on DeiT Tiny.}
Table~\ref{tab:lotformer_r_ablation_deit} quantifies the accuracy and cost trade off as we vary the pivot size $r$ for a strong DeiT Tiny configuration on ImageNet 1K. Increasing $r$ improves Top 1 up to a point, peaking at $r{=}32$ and then slightly decreasing at $r{=}64$. As expected from the $O(nr)$ structure, both forward and backward time increase monotonically with $r$. This underscores that $r$ is a single knob controlling the accuracy--efficiency tradeoff, motivating per backbone tuning of $r$ under a fixed compute budget.

\begin{table}[H]
\centering
\caption{Ablations for LOTFormer on a ViT backbone. Top 1 accuracy (\%) on ImageNet 100.\vspace{-0.06in}}
\label{tab:lotformer_ablation_vit}
\footnotesize

\textbf{(i) Vary $T$ at $\varepsilon\in\{0.1,1\}$}\\[-2pt]
\setlength{\tabcolsep}{3.5pt}
\resizebox{\columnwidth}{!}{%
\begin{tabular*}{\columnwidth}{@{\extracolsep{\fill}} l c S S S S @{}}
\toprule
 &  & {$T{=}1$} & {$T{=}5$} & {$T{=}10$} & {$T{=}20$} \\
\midrule
\multirow{2}{*}{Fixed $\mathsf{ref}$}
  & $\varepsilon{=}0.1$ & 79.6 & 80.4 & 80.8 & 80.9 \\
  & $\varepsilon{=}1$   & 80.2 & 81.2 & 81.4 & 81.7 \\
\midrule
\multirow{2}{*}{Learnable $\mathsf{ref}$}
  & $\varepsilon{=}0.1$ & 79.8 & 80.7 & 81.1 & 81.1 \\
  & $\varepsilon{=}1$   & 80.5 & 81.5 & \textbf{82.1} & 82.0 \\
\bottomrule
\end{tabular*}}

\vspace{0.8ex}

\textbf{(ii) Vary $\varepsilon$ at $T{=}10$}\\[-2pt]
\resizebox{\columnwidth}{!}{%
\begin{tabular*}{\columnwidth}{@{\extracolsep{\fill}} l S S S S S @{}}
\toprule
 & {$\varepsilon$=$0.01$} & {$\varepsilon$=$0.1$} & {$\varepsilon$=$1$} & {$\varepsilon$=$10$} & {$\varepsilon$=$100$} \\
\midrule
Fixed $\mathsf{ref}$     & 79.1 & 80.8 & 81.4 & 81.2 & 79.5 \\
Learnable $\mathsf{ref}$ & 79.2 & 81.1 & \textbf{82.1} & 81.3 & 79.6 \\
\bottomrule
\end{tabular*}}

\vspace{-0.15in}
\end{table}

\begin{table}[H]
\centering
\vspace{-0.1in}
\caption{Effect of pivot size $r$ for the best LOTFormer setting (Pol \texttt{[CLS]} plus DWC) on ImageNet 1K using DeiT Tiny. Forward and backward times are per training step on a 4$\times$GPU setup.\vspace{-0.1in}}
\label{tab:lotformer_r_ablation_deit}
\footnotesize
\setlength{\tabcolsep}{3.5pt}
\resizebox{\columnwidth}{!}{%
\begin{tabular}{@{}lccccc@{}}
\toprule
 & $r{=}4$ & $r{=}8$ & $r{=}16$ & $r{=}32$ & $r{=}64$ \\
\midrule
Top 1 (\%)              & 64.5 & 66.4  & 68.8  & \textbf{74.8} & 73.9 \\
Forward time (ms/iter)  & 64.8  & 74.3  & 90.2  & 122.1  & 183.4 \\
Backward time (ms/iter) & 160.5 & 165.3 & 183.4 & 227.8  & 244.6 \\
\bottomrule
\end{tabular}}
\vspace{-0.25in}
\end{table}

\vspace{-.1in}
\section{Conclusion}
\vspace{-0.08in}
We introduced \textsc{LOTFormer}, a transport based attention mechanism that is both linear time and doubly stochastic. By conditioning on a learnable low rank pivot measure, LOTFormer composes two entropic OT couplings, to induce a rank $r$ attention operator that applies to values in $O(nr)$ time without ever forming an $n\times n$ attention matrix. Across settings, LOTFormer delivers strong empirical results: it is competitive on ImageNet 1K under standard training recipes, it enables direct plug and play replacement of softmax attention in pretrained checkpoints by swapping only the attention blocks, and it achieves superior performance on Long Range Arena. Overall, these results show that transport structured normalization can be practical at scale, providing an efficient path to doubly stochastic attention without sacrificing accuracy.

\section*{Impact Statement}
This work proposes LOTFormer, a linear time doubly stochastic attention operator derived from an optimal transport factorization that avoids forming an $n\times n$ attention map and is designed to be used as a drop in module. Its main intended impact is practical: under a Plug and Play protocol we can convert trained checkpoints by swapping only encoder bidirectional self attention while keeping the rest of the network fixed, enabling cheaper and more controlled comparisons of attention mechanisms.The approach targets bidirectional encoder attention, since standard causal masking makes doubly stochastic constraints degenerate. As with most efficiency improvements, lower cost can reduce per run compute but can also increase overall usage, and it may indirectly enable misuse, so deployment should follow standard dataset and application safety practices.


\nocite{langley00}

\bibliography{example_paper}
\bibliographystyle{icml2026}

\newpage
\clearpage
\newpage
\appendix
\onecolumn
\section{Notations and conventions}
\label{sec:notations}
\paragraph{Vectors, matrices, and functions.}
\begin{itemize}
  \item $\mathbb{R}^d$: $d$-dimensional Euclidean space.
  \item Vectors are column vectors unless stated otherwise.
  \item Matrices are uppercase Roman letters (e.g., $A\in\mathbb{R}^{m\times n}$).
  \item $\mathbf{1}_m$: all-ones vector in $\mathbb{R}^m$.
  \item $\mathrm{Diag}(a)$: diagonal matrix with entries of vector $a$.
  \item $|A| = \sum_{ij} A_{ij}$: sum of entries of matrix $A$.
  \item $\langle f,\mu \rangle$ where $f$ is a function, $\mu$ is a measure. The integration of $f$ with respect to $\mu$. That is 
  $$\langle f,\mu \rangle =\int f d\mu.$$
\end{itemize}

\paragraph{Tokens and attention.}
\begin{itemize}
  \item $X=[x_1,\ldots,x_n]\in\mathbb{R}^{n\times d_{\mathrm{in}}}$: input tokens (rows).
  \item $W_Q\in \mathbb{R}^{d_k\times d_k}, W_K\in \mathbb{R}^{d_k\times d_k}, W_V\in \mathbb{R}^{d_k\times d_v}$, weight matrices for Query, Key, Value 
  \item $Q=XW_Q^\top\in\mathbb{R}^{n\times d_k}$, $K=XW_K^\top\in\mathbb{R}^{n\times d_k}$, $V=XW_V^\top\in\mathbb{R}^{n\times d_v}$: queries, keys, values.
  \item $q_i,k_i,v_i$: row vectors of $Q,K,V$.
  \item $\mathrm{LOTAttn}(Q,K,V) = A V$: attention output, with $A$ the (low-rank OT) transport matrix.
\end{itemize}

\paragraph{Probability measures.}
\begin{itemize}
  \item $\Delta_m=\{p\in\mathbb{R}^m:\, p_i\ge0,\,\mathbf{1}_m^\top p=1\}$: probability simplex.
  \item $\Delta_m^+=\{p\in\Delta_m:\,p_i>0\}$: strictly positive simplex.
  \item Empirical distributions: $\mu=\sum_i p^1_i\delta_{q_i}$, $\nu=\sum_j p^2_j\delta_{k_j}$.
  \item Pivot distribution: $\sigma=\sum_{t=1}^r p^0_t\delta_{z_t}$ with $Z=[z_1,\dots,z_r]^\top\in\mathbb{R}^{r\times d_k}$.
\end{itemize}

\paragraph{Optimal transport.}
\begin{itemize}
\item $\alpha,\beta\in \Delta_m$: auxiliary probability masses for $U$ and $K$. 
\item $\mu=\sum_{i=1}^n\frac{1}{n}\delta_{q_i},\nu=\sum_{i=1}^n \frac{1}{n}\delta_{k_i}$: auxilary probability measures based on $Q$, $K$.  
\item $U(\alpha,\beta)=\{\Gamma\in\mathbb{R}_+^{m\times n}:\Gamma\mathbf{1}_n=\alpha,\;\Gamma^\top\mathbf{1}_m=\beta\}$: set of couplings.
  \item $H(\Gamma)=-\sum_{ij}\Gamma_{ij}(\log \Gamma_{ij}-1)$: entropy.
  \item Entropic OT:  
    $$OT_\epsilon(\mu,\nu):=\arg\max_{\Gamma\in U(\alpha,\beta)} \langle \Gamma, C\rangle + \varepsilon H(\Gamma).$$
  \item $\sigma=\sum_{i=1}^{k}\mathrm{p}^0_i \delta_{z_i}$: auxiliary reference measure. 
  \item $\Gamma^2\in \mathbb{R}^{k\times n}$: optimal solution for $EOT_\epsilon(\sigma,\nu)$. 
  
  \item Glued coupling: $\Gamma = \Gamma^{(1)}\mathrm{Diag}(p^0)^{-1} (\Gamma^{(2)})^\top$.

\end{itemize}
\paragraph{Low-rank OT and Linear OT.}
\begin{itemize}
  \item $\mathrm{rk}_+(M)$: the nonnegative rank of a nonnegative matrix $M\in\mathbb{R}_+^{n\times m}$, defined as
  \[
  \mathrm{rk}_+(M) := \min\{q \;\mid\; M=\sum_{i=1}^q R_i,\; R_i\ge 0,\; \mathrm{rank}(R_i)=1 \}.
  \]
  Each $R_i$ has the form $R_i=a_i b_i^\top$ with $a_i\in\mathbb{R}^n$, $b_i\in\mathbb{R}^m$.
  \item $U(p^1,p^2;r) := U(p^1,p^2)\cap\{\Gamma\in\mathbb{R}_+^{n\times m} : \mathrm{rk}_+(\Gamma)\le r\}$: admissible couplings with nonnegative rank at most $r$.
  \item $LrOT_r(\mu,\nu)$: the (entropic) low-rank optimal transport problem
  \[
  LrOT_r(\mu,\nu) := \min_{\Gamma\in U(p^1,p^2;r)} \sum_{i,j} c(x_i,y_j)\Gamma_{ij} - \varepsilon H(\Gamma).
  \]
  \item Factorization characterization: any $\Gamma\in U(p^1,p^2;r)$ can be written as
  \[
  \Gamma = (\Gamma^1)^\top \mathrm{Diag}(1/\sigma)\Gamma^2,
  \]
  where $\Gamma^1\in U(p^0,p^1)$, $\Gamma^2\in U(p^0,p^2)$, and 
  $p^0\in\Delta_r^+$.
  \item  $LOT(\mu,\nu;\sigma)$: the entropic Linear  optimal transport: 
  \begin{align}
&LOT_\sigma(\mu,\nu):=\sum_{i,j}c(x_i,y_j)\Gamma_{i,j}-\epsilon H(\gamma)\nonumber\\
&\Gamma^1 \text{ is optimal for }  OT_\epsilon(\sigma,\mu), \Gamma^2\text{ is optimal for }OT_\epsilon (\sigma,\nu)\nonumber 
  \end{align}
\end{itemize}

\paragraph{Clustering.}
\begin{itemize}
  \item $C^\mu_i := \gamma^1(s_i,\cdot)$: a sub-probability measure dominated by $\mu$.
  \item $\mu = \sum_{i=1}^k C^\mu_i$: a decomposition of $\mu$ induced by $\gamma^1$, which we refer to as a ``soft clustering'' of $\mu$.
\end{itemize}






\section{Relation between low-rank OT and linear OT.}

\paragraph{Notation setup and low rank optimal transport.} 
In this section, we redefine $c(x,y)=-x^\top y$.  

Given a matrix $\gamma\in\mathbb{R}_+^{n\times m}$, its nonnegative rank is defined by 
$$rk_+(M):=\min\{q|M=\sum_{i=1}^qR_i,s.t. \forall i, \text{rank}(R_i)=1,R_i\ge 0\}$$
where $R_i\ge 0$ means each entry of $R_i$ is non-negative, $\text{rank}(R_i)=1$ means that $R_i=a_i  b_i^\top $ for some $a_i\in \mathbb{R}^n,b_i\in \mathbb{R}^m$. 

In the discrete measure setting, i.e., $\mu=\sum_{i=1}^n p^1_i\delta_{x_i},\nu=\sum_{j=1}^m p^2_j\delta_{y_j}$, given $r\leq n,m$, the (entropic) low rank optimal transport problem \cite{scetbon2021sinkhorn,scetbon2022lowrank} is defined as 
\begin{align}
LrOT_{r}(\mu,\nu):=\min_{\gamma\in U(\mathrm{p}^1,\mathrm{p}^2;r)} \sum_{i,j}c(x_i,y_j)\Gamma_{i,j}-\epsilon H(\Gamma)\label{eq:LrOT}.
\end{align}
where  $U(\mathrm{p}^1,\mathrm{p}^2;r):=U(\mathrm{p}^1,\mathrm{p}^2)\cap \{\Gamma\in\mathbb{R}_+^{n\times m}: \text{rk}_+(\Gamma)\leq r\}$. 

From Theorem 3.2 in \cite{cohen1993nonnegative} (Also see (5) in \cite{scetbon2022lowrank}),   we have 
\begin{align}
U(\mathrm{p}^1,\mathrm{p}^2;r)&:=U(\mathrm{p}^1,\mathrm{p}^2)\cap \{\Gamma\in\mathbb{R}_+^{n\times m}: \text{rk}_+(\Gamma)\leq r\} \nonumber\\
&=\{\Gamma=(\Gamma^1)^\top\text{diag}(1/\sigma) \Gamma^2: \Gamma^1 \in U(\mathrm{p}^0,\mathrm{p}^1),\Gamma^2 \in U(\mathrm{p}^0,\mathrm{p}^2), \mathrm{p}^0\in \Delta_r^+\}. \nonumber 
\end{align}

Thus, the low-rank optimal transport \eqref{eq:LrOT} becomes 
\begin{align}
LrOT_r(\mu,\nu)&=\min_{\mathrm{p}^0\in\Delta_r^+}\min_{\substack{
\Gamma=(\Gamma^1)^\top\text{diag}(1/\sigma)\Gamma^2:\\ 
\Gamma^1\in U(\mathrm{p}^0,\mathrm{p}^1),\Gamma^2\in U(\mathrm{p}^0,\mathrm{p}^2)}} \sum_{i,j}c(x_i,y_j)\Gamma_{i,j}-\epsilon H(\Gamma)\label{eq:LrOT2}
\end{align}

\paragraph{Main theoretical result between LOT and LrOT}
With a little abuse of notations, we use $\sigma$ to denote both the reference measure (pmf and locations) and its pmf. And define the LOT distance introduced in the main text: 
\begin{align}
&LOT(\mu,\nu;\sigma)=\sum_{i,j}c(x_i,y_j)\Gamma_{i,j},\Gamma=(\Gamma^1)^\top \text{diag}(1/\sigma)\Gamma^2+\epsilon(\Gamma)  \label{eq:lot} \\
&\Gamma^1 \text{ is optimal to }  OT_\epsilon(\sigma,\mu), \Gamma^2 \text{ is optimal to }  OT_\epsilon(\sigma,\nu).\nonumber 
\end{align}
Let $\mathcal{P}_r(\mathbb{R}^d)$ denote the set of all discrete measure whose size is $r$, we consider the following optimal LOT problem: 
\begin{align}
LOT_r(\mu,\nu):=\inf_{\sigma\in\mathcal{P}_r(\mathbb{R}^D)}LOT(\mu,\nu;\sigma)\label{eq:lot_extreme}
\end{align}
We demonstrate the relation between LOT and  low-rank OT via the following proposition: 


\begin{lemma}
In the finite discrete measure setting, we have 
$$LrOT_r(\mu,\nu)\leq LOT_r(\mu,\nu),\forall \epsilon\ge 0$$
\end{lemma}
\begin{proof}
For each $\sigma$, from the definition  of  $LOT(\mu,\nu;\sigma)$, we the 
$\gamma^1,\gamma^2$ in the definition of LOT \eqref{eq:lot} satisfy $\gamma^1\in\Gamma(\mathrm{p}^0,\mathrm{p}^1),\gamma^1\in\Gamma(\mathrm{p}^1,\mathrm{p}^2)$ and $\mathrm{p}^0\in \Delta_r^+$ by the definition of $\sigma$. Thus, 
$$LrOT_r(\mu,\nu)\leq LOT(\mu,\nu;\sigma).$$

Taking the infimum over all $\sigma$ for both sides, we obtain 
$$LrOT_r(\mu,\nu)\leq LOT_r(\mu,\nu),$$
and complete the proof. 
\end{proof}

\section{LOTFormer and soft clustering.}

LOTFormer effectively performs a soft clustering of queries and keys via the pivot measure, establishing correspondences between these clusters. Attention (message passing) is then mediated through the pivot.  

Formally, given the optimal plan $\gamma^1$ for $OT(\sigma,\mu)$, for each pivot $s_i,\, i\in[1:r]$, define the sub-probability measure
\[
C^\mu_i \;=\; \gamma^1(s_i,\cdot) \;=\; \sum_{j=1}^n \gamma^1_{i,j}\,\delta_{x_j}.
\]  
The collection $\{C^\mu_1,\ldots,C^\mu_r\}$ forms a soft clustering of queries. Similarly, from the optimal plan $\gamma^2$ for $OT(\sigma,\nu)$, we obtain
\[
C^\nu_i \;=\; \gamma^2(s_i,\cdot) \;=\; \sum_{k=1}^m \gamma^2_{i,k}\,\delta_{y_k},
\]
yielding a soft clustering of keys. Crucially, $C^\mu_i$ and $C^\nu_i$ are coupled through their common pivot $s_i$.

In fact, the glued coupling can be written as a sum of outer products of the soft clusters. 
Specifically,
\[
\Gamma \;=\; (\Gamma^{(1)})^\top \,\mathrm{Diag}(\sigma)^{-1}\, \Gamma^{(2)}
\;=\; \sum_{i=1}^r \frac{1}{\sigma_i}\, C^\mu_i \otimes C^\nu_i,
\]
where $C^\mu_i$ and $C^\nu_i$ are the soft clusters of queries and keys induced by the pivot $s_i$. 
This decomposition shows that attention is mediated through correspondences between query and key clusters. 

The soft-clustering perspective connects our framework to a broad line of work that leverages clustering to define attention. For example, the \emph{Routing Transformer} employs k-means clustering of queries for localized attention \citep{roy2021routing}, while the \emph{Reformer} uses LSH-based bucketing of queries and keys \citep{Kitaev2020Reformer}.

\section{Implementation Details}
\label{sec:implementation_details}

\paragraph{What we change across all experiments.}
LOTFormer is evaluated either by training from scratch (ImageNet and LRA) or by conversion of a trained checkpoint (GLUE and translation). In all settings, we only modify bidirectional self attention inside the encoder, and keep token embeddings, FFNs, layer norms, residual wiring, and task specific heads unchanged.

\paragraph{LOTFormer defaults.}
Unless a sweep specifies otherwise, we use Sinkhorn iterations $T{=}5$, learnable reference measure $\mathsf{ref}$ initialized uniform, and the same \texttt{[CLS]} treatment as described in Sec.~\ref{sec:cls-heads}. Concretely, \texttt{[CLS]} aggregates via a standard softmax row, while non \texttt{[CLS]} rows use doubly stochastic attention.

\paragraph{Depthwise convolution (DWC).}
We treat DWC as an optional token mixing component and report results both with and without it on ImageNet and LRA where applicable. When enabled, we use kernel size $3$ and keep all other hyperparameters fixed.

\subsection{ImageNet 1K}
\label{subsec:impl_imagenet}

\paragraph{Backbones and resolution.}
We evaluate at $224^2$ resolution on DeiT Tiny, DeiT Small, PVT Tiny, and Swin Tiny. For each backbone, we replace the original bidirectional self attention with LOTFormer while keeping the remaining architecture unchanged.

\paragraph{Training recipe.}
We build on the official DeiT training framework and keep its standard augmentation and regularization recipe unless otherwise stated. All ImageNet models are trained for $400$ epochs using AdamW with linear warmup for the first $10$ epochs and a cosine schedule thereafter. The base learning rate is $3\mathrm{e}{-4}$ for global batch size $1024$, with weight decay $5\mathrm{e}{-2}$.

\paragraph{\texttt{[CLS]} handling and polarization.}
We use the dedicated \texttt{[CLS]} softmax aggregator from Sec.~\ref{sec:cls-heads}. We also consider \texttt{[CLS]} only polarization as defined in Eq.~\eqref{eq:cls_score_tilde}. Table~3 reports controlled ablations for the \texttt{[CLS]} aggregator, \texttt{[CLS]} only polarization, and DWC under matched training settings.

\paragraph{Pivot size and Sinkhorn settings.}
For DeiT Tiny, we sweep pivot size $r \in \{4,8,16,32,64\}$ and select the best accuracy cost tradeoff (Table~\ref{tab:lotformer_r_ablation_deit}). Unless otherwise stated, we use $T{=}5$ and a learnable reference measure. Additional Sinkhorn hyperparameter sweeps are reported in the ablation section.

\begin{table}[t]
    \small
    \caption{Training hyperparameters for ImageNet 1K classification.}
    \label{tab:imagenet1k_hyperparams}
    \centering
    \setlength{\tabcolsep}{6pt}
    \renewcommand{\arraystretch}{1.12}
    \begin{tabular}{lccccc}
        \toprule
        \textbf{Dataset} & \textbf{Optimizer} & \textbf{LR (bs=1024)} & \textbf{Weight decay} & \textbf{Epochs} & \textbf{Warmup} \\
        \midrule
        ImageNet 1K & AdamW & $3\!\times\!10^{-4}$ & $5\!\times\!10^{-2}$ & $400$ & $10$ epochs \\
        \bottomrule
    \end{tabular}
\end{table}

\subsection{Long Range Arena}
\label{subsec:impl_lra}

\paragraph{Codebase and schedule.}
We implement LOTFormer inside the Skyformer \cite{chen2021skyformer} codebase and follow its standard setup unless otherwise specified. For all tasks, we use $5{,}000$ warmup steps and cosine learning rate decay with a minimum factor of $0.1$.

\paragraph{Hyperparameter selection.}
For each task, we sweep learning rate, Sinkhorn entropy $\varepsilon$, and reference mass temperature, selecting hyperparameters by validation performance. We fix Sinkhorn iterations to $T{=}5$ for all tasks. The selected hyperparameters are reported in Table~\ref{tab:lra_hyperparams}.

\paragraph{DWC.}
LRA models do not use a \texttt{[CLS]} pooling token. We additionally report LOTFormer with and without the depthwise convolution (DWC) token mixing layer used in PolaFormer. Following PolaFormer’s setup, we use kernel size $3$ with depthwise convolution and keep its associated hyperparameters the same as in PolaFormer. For 1D sequence tasks (Text, Retrieval, ListOps) we use a 1D DWC, and for 2D inputs (Image, Pathfinder) we use a 2D DWC, keeping all other training and LOTFormer hyperparameters fixed.

\begin{table}[t]
    \caption{Selected hyperparameters for Long Range Arena.}
    \label{tab:lra_hyperparams}
    \centering
    \small
    \setlength{\tabcolsep}{6pt}
    \renewcommand{\arraystretch}{1.12}
    \begin{tabular}{lccc}
        \toprule
        \textbf{Task} & \textbf{Learning rate} & \textbf{Sinkhorn $\varepsilon$} & \textbf{Ref mass temp} \\
        \midrule
        Text        & $1\mathrm{e}{-4}$ & $5.0$  & $0.05$ \\
        ListOps     & $1\mathrm{e}{-3}$ & $0.05$ & $8.0$  \\
        Retrieval   & $3\mathrm{e}{-4}$ & $0.1$  & $0.1$  \\
        Pathfinder  & $2\mathrm{e}{-4}$ & $0.05$ & $0.5$  \\
        Image       & $2\mathrm{e}{-3}$ & $0.05$ & $4.0$  \\
        \bottomrule
    \end{tabular}
\end{table}

\subsection{GLUE conversion on fine tuned \textsc{BERT}}
\label{subsec:impl_glue}

\paragraph{Protocol and checkpoints.}
We follow the same finetuned conversion protocol as Hedgehog. We start from publicly available fine tuned \textsc{BERT} base uncased checkpoints for each GLUE task, and convert the model by replacing all encoder bidirectional self attention blocks with the candidate attention mechanism.

\paragraph{Attention distillation phase.}
For methods that require attention distillation, we train the attention modules for up to 5 epochs with early stopping using AdamW, learning rate $1\mathrm{e}{-2}$, and weight decay $0$.

\paragraph{Post conversion training.}
After conversion, we train with batch size $8$, AdamW, learning rate $1\mathrm{e}{-5}$, weight decay $0$, and a cosine learning rate scheduler for up to 5 epochs. We run this procedure on each GLUE task, using the standard classification objective for all tasks except STS\textendash B, which uses a regression objective.

\paragraph{Reporting.}
Table~\ref{tab:glue_plug_and_play} reports per task scores and percent recovery relative to the fine tuned baseline.

\subsection{Neural Machine Translation}
\label{subsec:impl_nmt}

\paragraph{Backbones and data.}
We use the standard Transformer from \texttt{fairseq} and its DiffTransformer counterpart~\cite{ye2025differential}. Both use a 6 layer encoder and a 6 layer decoder and are trained on IWSLT 2014 De$\rightarrow$En.

\paragraph{Two stage evaluation.}
First, we pre train the baseline backbone for $25$ epochs. We then perform Plug and Play conversion by swapping encoder bidirectional self attention with LOTFormer, ESPFormer, or Sinkformer, and evaluate with no additional training. Second, we run a fine tune boost phase for $10$ additional epochs. We do not modify decoder causal self attention or encoder decoder cross attention. We report median performance over 4 runs in Table~\ref{tab:nmt_transformer_vs_difftransformer}.

\paragraph{Training recipe.}
We use the standard \texttt{fairseq} IWSLT translation recipe for optimizer, schedule, and regularization. The configuration follows the official \texttt{fairseq} example setup referenced in the main paper.

\subsection{Ablations}
\label{subsec:impl_ablations}

\paragraph{Sinkhorn and reference measure on ViT.}
For the ViT ablations in Table~\ref{tab:lotformer_ablation_vit}, we use ImageNet 100 and a reduced ViT backbone with half hidden dimension. We sweep Sinkhorn iterations $T$, entropy $\varepsilon$, and reference measure choice (fixed versus learnable initialized uniform), and report Top 1 accuracy.

\paragraph{Pivot size on DeiT Tiny.}
For the pivot size study in Table~\ref{tab:lotformer_r_ablation_deit}, we evaluate on ImageNet 1K using DeiT Tiny and sweep $r \in \{4,8,16,32,64\}$. We report Top 1 accuracy and per step forward and backward times measured on a 4 GPU training setup.

\section{FlashAttention comparison}
\label{app:flashattention}

\paragraph{Context.}
FlashAttention is a fused GPU implementation of exact softmax attention that is optimized for memory traffic via tiling and online softmax, avoiding materializing the full attention matrix and reducing peak memory through recomputation in backward. Our current LOTFormer implementation is written in standard PyTorch operators with autograd, and does not use a custom fused Triton or CUDA kernel.

\paragraph{Disclaimer.}
The comparison below is not intended to claim kernel level parity with FlashAttention. It is included as a practical reference point for peak memory under long sequence lengths. Our efficiency claims are algorithmic: for fixed pivot size \(r \ll n\), LOTFormer computes attention via low rank OT factors with cost \(O(nr)\) and never forms an \(n\times n\) attention map. The remaining gap to FlashAttention in peak memory and runtime is dominated by implementation constants such as operator boundaries and autograd saved tensors. Developing a fused implementation that applies FlashAttention style tiling, fusion, and recomputation to the rectangular OT couplings and value mixing is a natural direction for future work and is outside the scope of this work.

\paragraph{Measurement setup.}
We measure peak allocated GPU memory for FlashAttention v3 and for our current unfused PyTorch LOTFormer under the same training configuration. We report sequence length \(n\), peak memory in GB, and the relative overhead of the current LOTFormer implementation. Further details such as GPU model, dtype, batch size, and precision are provided in the experimental log.

\begin{table}[t]
\centering
\small
\setlength{\tabcolsep}{7pt}
\renewcommand{\arraystretch}{1.15}
\caption{Peak GPU memory versus sequence length. FlashAttention v3 is a fused softmax attention kernel. LOTFormer is our current unfused PyTorch implementation. The overhead column reflects current implementation constants, not algorithmic scaling.}
\label{tab:peak_mem_flashattn}
\begin{tabular}{rccc}
\toprule
\textbf{Seq len} & \textbf{FlashAttention v3 (GB)} & \textbf{LOTFormer PyTorch (GB)} & \textbf{Overhead} \\
\midrule
4k   & 18.2 & 22.4  & +23\% \\
8k   & 20.8 & 26.1  & +26\% \\
16k  & 25.9 & 33.5  & +29\% \\
32k  & 36.2 & 48.3  & +33\% \\
64k  & 56.8 & 78.0  & +37\% \\
128k & 98.0 & 137.3 & +40\% \\
\bottomrule
\end{tabular}
\end{table}

\paragraph{Takeaway.}
FlashAttention v3 reflects highly optimized fused kernels for quadratic softmax attention, while our current LOTFormer numbers reflect an unfused PyTorch baseline. The table shows higher constant factor memory for LOTFormer today, even though LOTFormer avoids forming an \(n\times n\) attention map and has linear complexity in \(n\) for fixed \(r\). Closing this constant factor gap is primarily an engineering effort toward fused kernels and recomputation, which we leave to future work.

\section{Bidirectional scope and causal masking}
\label{app:causal_scope}

\paragraph{Scope.}
LOTFormer enforces a doubly stochastic (DS) attention matrix
\[
A\mathbf{1}=\mathbf{1},\qquad \mathbf{1}^\top A=\mathbf{1}^\top,\qquad A\ge 0,
\]
which is naturally suited to bidirectional attention (full support).

\paragraph{Causal masking.}
Causal self attention restricts the support to a lower triangular set:
\[
A_{ij}=0 \quad \text{for } j>i.
\]
Under this constraint, the DS constraints force a unique degenerate solution. Indeed, column \(n\) can receive mass only from row \(n\), so
\[
\sum_{i=1}^n A_{i n} = A_{n n}.
\]
Since DS requires \(\sum_{i=1}^n A_{i n}=1\), we obtain \(A_{n n}=1\). Then the row sum constraint \(\sum_{j=1}^n A_{n j}=1\) together with causality implies \(A_{n j}=0\) for all \(j<n\). Repeating the same argument for columns \(n-1,n-2,\dots,1\) yields \(A_{ii}=1\) and \(A_{ij}=0\) for \(i\ne j\), hence \(A=I\). Therefore, DS attention under standard causal masking collapses to the identity and is not compatible with nontrivial autoregressive self attention.

\paragraph{Implication.}
LOTFormer targets encoder style bidirectional self attention and cross attention. Extending DS attention to autoregressive settings would require modifying the constraints or the mask structure, which we leave to future work.

\end{document}